\documentclass[runningheads]{llncs}

\usepackage{amsmath}
\usepackage{amssymb}
\usepackage[T1]{fontenc}
\usepackage[utf8]{inputenc}
\usepackage{courier} 

\usepackage[colorlinks,hyperindex,bookmarks,linkcolor=blue,citecolor=blue,urlcolor=blue]{hyperref}
\usepackage{alltt}
\usepackage{makeidx}  
\usepackage{listings}
\usepackage{caption}
\usepackage{subfig}  
\usepackage{mathpartir}
\usepackage{todonotes}
\usepackage{wrapfig}
\usepackage{framed}
\usepackage{mdframed}   
\usepackage{mathtools}
\usepackage{stmaryrd}
\usepackage{soul}
\usepackage{tcolorbox}  

\lstdefinelanguage{L4}
{morekeywords={
      assert 
    , class  
    , decl   
    , defn   
    , extends
    , lexicon
    , fact   
    , rule   
    , derivable
    , let   
    , in    
    , not   
    , forall
    , exists
    , if   
    , then 
    , else 
    , for  
    , true 
    , false
},    
sensitive=false,
morecomment=[l]{\#},
morestring=[b]",
}

\usepackage{tikz}
\usetikzlibrary{trees}
\usetikzlibrary{arrows}
\usetikzlibrary{decorations.pathmorphing}
\usetikzlibrary{shapes.multipart}
\usetikzlibrary{shapes.geometric}
\usetikzlibrary{calc}
\usetikzlibrary{positioning} 
\usetikzlibrary{fit}
\usetikzlibrary{backgrounds}

\definecolor{dkgreen}{rgb}{0,0.6,0}
\definecolor{gray}{rgb}{0.5,0.5,0.5}
\definecolor{mauve}{rgb}{0.58,0,0.82}

\newcommand{\blue}[1]{{\color{blue}#1}}

\newcommand{\ie}{\textit{i.e. }}

\newcommand{\wrt}{\textit{w.r.t.\ }}

\newcommand{\appref}[1]{Appendix~\ref{#1}}
\newcommand{\secref}[1]{Section~\ref{#1}}
\newcommand{\secrefs}[1]{Sections~\ref{#1}}
\newcommand{\figref}[1]{Figure~\ref{#1}}

\newcommand{\lemmaref}[1]{Lemma~\ref{#1}}
\newcommand{\exampleref}[1]{Example~\ref{#1}}

\newcommand{\IMPL}[0]{\longrightarrow}
\newcommand{\AND}[0]{\land}
\newcommand{\OR}[0]{\lor}
\newcommand{\NOT}[0]{\lnot}

\newcommand{\IFF}[0]{\leftrightarrow}

\begin{document}
\title{Automating Defeasible Reasoning in Law}

\author{How Khang Lim\orcidID{0000-0002-9333-1364} \and
  Avishkar Mahajan\orcidID{0000-0002-9925-1533} \and
  Martin Strecker\orcidID{0000-0001-9953-9871} \and
  Meng Weng Wong\orcidID{0000-0003-0419-9443}
  }
  \institute{Singapore Management University}
\maketitle

\begin{abstract}
The paper studies defeasible reasoning in rule-based systems, in particular
about legal norms and contracts. We identify rule modifiers that specify how
rules interact and how they can be overridden. We then define rule
transformations that eliminate these modifiers, leading in the end to 
a translation of rules to formulas. For reasoning with and about rules, we
contrast two approaches, one in a classical logic with SMT solvers as proof
engines, one in a non-monotonic logic with Answer Set Programming solvers.


\end{abstract}



\section{Introduction}\label{sec:introduction}

Computer-supported reasoning about law is a longstanding effort of researchers
from different disciplines such as jurisprudence, artificial intelligence, logic and
philosophy. What originally may have appeared as an academic playground is
now evolving into a realistic scenario, for various reasons. 

On the \emph{demand} side, there is a growing number of human-machine or
machine-machine interactions where compliance with legal norms or with a contract is essential,
such as in sales, insurance, banking and finance or digital rights management, to name but
a few. Innumerable ``smart contract'' languages attest to the interest to
automate these processes, even though many of them are dedicated
programming languages rather than formalisms intended to express and reason about
regulations.

On the \emph{supply} side, decisive advances have been made in fields such as
automated reasoning and language technologies, both for computerised domain
specific languages (DSLs) and natural languages. Even though a completely
automated processing of traditional law texts capturing the subtleties of
natural language is currently out of scope, one can expect to code a law text
in a DSL that is amenable to further processing.

This ``rules as code'' approach is the working hypothesis of our CCLAW
project\footnote{\url{https://cclaw.smu.edu.sg/}}: law texts are formalised in
a DSL called L4 that is sufficiently precise to avoid ambiguities of natural
languages and at the same time sufficiently close to a traditional law text
with its characteristic elements such as cross references, prioritisation of
rules and defeasible reasoning. Indeed, presenting these features is one of
the main topics of this paper. Once a law has been coded in L4, it can
be further processed: it can be converted to natural language \cite{listemnmaa2021cnl} to
be as human-readable as a traditional law text, and efficient executable code can be
extracted, for example to perform tax calculations (all this is not the topic
of the present paper). It can also be analysed, to find faults in the law
text on the meta level (such as consistency and completeness of a rule set),
but also on the object level, to decide individual cases.


\paragraph{Overview of the paper}
The main emphasis of this paper is on the L4 DSL that is currently under
definition, which in particular features a formalism for transcribing rules
and reasoning support for verifying their properties. \secref{sec:l4_language} is dedicated to a description of the language
and the L4 system implementation currently under development. The rule
language will be dissected in \secref{sec:resasoning_with_rules}. We will in
particular describe mechanisms for prioritisation and defeasibility of rules
that are encoded via specific keywords in law texts. We then define a precise
semantics of these mechanisms, by a translation to logic. Classical, monotonic
logic, developed in \secref{sec:defeasible_classical}, has received surprisingly
little attention in this area, even though proof support in the form of
SAT/SMT solvers has made astounding progress in recent years. An alternative
approach, based on Answer Set Programming, is described in
\secref{sec:defeasible_asp}. We conclude in \secref{sec:conclusions}. \emph{Added by Avishkar} The guiding principle that underlies the

\paragraph{Related work}

There is a huge body of work both on computer-assisted legal reasoning and
(not necessarily related) defeasible reasoning. In a seminal work, Sergot and Kowalski
\cite{sergot_kowalski_etal__british_nationality_acm_1986,kowalski_legislation_logic_programs_1995}
code the British Nationality Act in Prolog, exploiting Prolog's negation as
failure for default reasoning.

The Catala language \cite{merigoux_chataing_protzenko_cata_icfp_2021},
extensively used for coding tax law and resembling more a high-level
programming language than a reasoning formalism, includes default rules, which
are however not entirely disambiguated during compile time so that run time
exceptions can be raised.

An entirely different approach to tool support is taken with the LogiKEy \cite{benzmueller_etal_logikey_2020}
workbench that codes legal reasoning in the Isabelle interactive proof assistant, paving
the way for a very expressive formalism. In contrast, we have opted for a DSL
with fully automated proofs which are provided by SMT respectively ASP solvers. These do not permit for human intervention in the proof process, which would not be adequate for the user group we target. Symboleo
\cite{sarifi_parvizimosaed_amyot_logrippo_mylopoulos_Symboleo_spec_legal_contracts_2020}
and the NAI Suite
\cite{libal_steen_nai_suite_draft_reason_legal_texts_jurix_2019} emphasise
deontic logic rather than defeasible reasoning (the former is so far not
considered in our L4 version).

As a result of a long series of logics, see for example
\cite{governatori_carnead_defeas_logic_icail_2011,governatori21:_unrav_legal_refer_defeas_deont_logic},
Governatori and colleagues have developed the Turnip
system\footnote{\url{https://turnipbox.netlify.com/}} that is based on a
combination of defeasible and deontic logic. The system is applied, among
others, to modelling traffic rules
\cite{governatori_Traffic_Rules_Encoding_using_Defeasible_jurix_2020}.

It seems vain to attempt an exhaustive
review of defeasible reasoning. Before the backdrop of foundational law
theory \cite{hart_concept_of_law_1997}, there are sometimes diverging
proposals for integrating defeasibility, sometimes opting for non-monotonic
logics \cite{hage_law_and_defeasibility_2003}, sometimes taking a more
classical stance \cite{alchourron_makinson_hierarchies_of_regulations_1981}. 
Defeasible rule-based reasoning in the context of argumentation theory is
discussed in \cite{dung_argumentation_theory_1995,amgoud_besnard_rule_based_argumentation_systems_2019}.

On a more practical side, Answer Set Programming (ASP) \cite{asp_background}
goes beyond logic programming and increasingly integrates techniques from
constraint solving, such as in the sCASP system
\cite{arias_phd_2019}. In spite of a convergence of SMT and CASP technologies,
there are few attempts to use SMT for ASP, see
\cite{shen_lierler_smt_answer_set_kr_2018}. For the technologies used in our
own implementation, please see \secref{sec:conclusions}.


\section{An Overview of the L4 Language}\label{sec:l4_language}

This section gives an account of the L4 language as it is currently defined --
as an experimental language, L4 will evolve over the next months. In our
discussion, we will ignore some features such as a natural language interface
\cite{listemnmaa2021cnl} which are not relevant for the topic of this paper but are relevant in the future.

As a language intended for representing legal texts and reasoning about them,
an L4 module is essentially composed of four sections:
\begin{itemize}
\item a terminology in the form of \emph{class definitions};
\item \emph{declarations} of functions and predicates;
\item \emph{rules} representing the core of a law text, specifying what is
  considered as legal behaviour;
\item \emph{assertions} for stating and proving properties about the rules.
\end{itemize}

For lack of space, we cannot give a full description here -- see
\appref{sec:l4_language_app} for more details. 
We will illustrate the concepts with an example, a (fictitious) regulation of speed
limits for different types of vehicles. Classes are, for example, \texttt{Car}
and its subclass \texttt{SportsCar}, \texttt{Day} and \texttt{Road}. We will
in particular be interested in specifying the maximal speed \texttt{maxSp} of
a vehicle on a particular day and type of road, and this will be the purpose
of the rules.

Before discussing rules, a few remarks about \emph{expressions} which are
their main constituents: L4 supports a simple functional language featuring
typical arithmetic, Boolean and comparison operators, an \texttt{if .. then
  .. else} expression, function application, anonymous functions (\ie, lambda
abstraction) written in the form \texttt{$\backslash$x : T -> e}, class
instances and field access. A \emph{formula} is just a
Boolean expression, and, consequently, so are quantified formulas
\texttt{forall x:T. form} and \texttt{exists x:T. form}.

In its most complete form, a \emph{rule} is composed of a list of variable
declarations introduced by the keyword \texttt{for}, a precondition introduced
by \texttt{if} and a post-condition introduced by
\texttt{then}. \figref{fig:rules} gives an example of rules of our speed limit
scenario, stating, respectively, that the maximal speed of cars is 90 km/h on a
workday,
and that they may drive at 130 km/h if the road is a highway.  Note that in
general, both pre- and post-conditions are Boolean formulas that can be
arbitrarily complex, thus are not limited to conjunctions of literals in the
preconditions or atomic formulas in the post-conditions.

\begin{figure}[h!]
  \begin{lstlisting}
rule <maxSpCarWorkday> 
   for v: Vehicle, d: Day, r: Road
   if isCar v && isWorkday d
   then maxSp v d r 90
rule <maxSpCarHighway>
   for v: Vehicle, d: Day, r: Road
   if isCar v && isHighway r
   then maxSp v d r 130
\end{lstlisting}
  \caption{Rules of speed limit example}\label{fig:rules}
\end{figure}

Rules whose precondition is \texttt{true} can be written as \texttt{fact}
without the \texttt{if} \dots \texttt{then} keywords.
Rules may not contain free variables, so all variables occurring in the body of
the rule have to be declared in the \texttt{for} clause. In the absence of
variables to be declared, the \texttt{for} clause can be omitted. 
Intuitively, a rule
\begin{lstlisting}[frame=none,mathescape=true]
  rule <r> for $\overrightarrow{v}$: $\overrightarrow{T}$ if Pre $\overrightarrow{v}$ then Post $\overrightarrow{v}$
\end{lstlisting}
corresponds to a universally quantified formula
$\forall \overrightarrow{v} : \overrightarrow{T}.\; Pre \overrightarrow{v}
\IMPL Post \overrightarrow{v}$ that could directly be written as a fact,
and it may seem that a separate rule syntax is redundant. This is not so,
because the specific structure of rules makes them amenable to transformations
that are useful for defeasible reasoning, as will be seen in
\secrefs{sec:resasoning_with_rules} and \ref{sec:defeasible_classical}.

Apart from user-defined rules and rules obtained by transformation, there are
system generated rules: For each subclass relation \texttt{$C$ extends $B$}, a
class inclusion axiom of the form \texttt{for x: $S$ if is$C$ x then is$B$ x}
is generated, where \texttt{is$C$} and \texttt{is$B$} are the characteristic
predicates and $S$ is the common super-sort of $C$ and $B$.

\begin{figure}[h]
\begin{lstlisting}
assert <maxSpFunctional> {SMT: {valid}}
   maxSp instCar instDay instRoad instSpeed1 &&
   maxSp instCar instDay instRoad instSpeed2
   --> instSpeed1 == instSpeed2
\end{lstlisting}
  \caption{Assertions of speedlimit example}\label{fig:assertions}
\end{figure}

The purpose of our formalization efforts is to be able to make assertions
and prove them, such as the statement in \figref{fig:assertions} which claims
that the predicate 
\texttt{maxSp} behaves like a function, \ie{} given the same car, day and
road, the speed will be the same. Instead of a universal quantification, we
here use variables \texttt{inst...} that have been declared globally, because they
produce more readable (counter-)models.


\section{Reasoning with and about Rules}\label{sec:resasoning_with_rules}

Given a plethora of different notions of ``defeasibility'', we had to make a
choice as to which notions to support, and which semantics to give to them. We
will here concentrate on two concepts, which we call \emph{rule modifiers},
that limit the applicability of rules and make them ``defeasible''. They will
be presented informally in the following. Giving them a precise semantics in
classical, monotonic logic is the topic of \secref{sec:defeasible_classical};
a semantics based on Answer Set Programming will be provided in
\secref{sec:defeasible_asp}.

We will concentrate on two rule modifiers that restrict the applicability of
rules and that frequently occur in law texts: \emph{subject to} and
\emph{despite}. A motivating discussion justifying their informal semantics is
given in \appref{sec:resasoning_with_rules_app}, drawn from a detailed
analysis of Singapore's Professional Conduct Rules. In the following, however,
we return to our running example.

\begin{example}
  The rules \texttt{maxSpCarHighway} and \texttt{maxSpCarWorkday} are not
  mutually exclusive and contradict another because they postulate different
  maximal speeds. For disambiguation, we would like to say:
  \texttt{maxSpCarHighway} holds \emph{despite} rule
  \texttt{maxSpCarWorkday}. In L4, rule modifiers are introduced with the aid
  of \emph{rule annotations}, with a list of rule names following the keywords
  \texttt{subjectTo} and \texttt{despite}. Thus, we modify rule
  \texttt{maxSpCarHighway} of \figref{fig:rules} with
\begin{lstlisting}
rule <maxSpCarHighway>
  {restrict: {despite: maxSpCarWorkday}}
# rest of rule unchanged
\end{lstlisting}
Furthermore, to the delight of the public of the country with the highest
density of sports cars, we also introduce a new rule \texttt{maxSpSportsCar}
that holds \emph{subject to} \texttt{maxSpCarWorkday} and \emph{despite}
\texttt{maxSpCarHighway}:
\begin{lstlisting}
rule <maxSpSportsCar>
  {restrict: {subjectTo: maxSpCarWorkday, 
              despite: maxSpCarHighway}}
   for v: Vehicle, d: Day, r: Road
   if isSportsCar v && isHighway r
   then maxSp v d r 320
 \end{lstlisting}
\end{example}

We will now give an informal characterization of these modifiers:
\begin{itemize}
\item $r_1$ \emph{subject to} $r_2$ and $r_1$ \emph{despite} $r_2$ are complementary
  ways of expressing that one rule may override the other rule. They have in
  common that $r_1$ and $r_2$ have contradicting conclusions. The conjunction
  of the conclusions can either be directly unsatisfiable (such as: ``may hold'' vs.{}
  ``must not hold'') or unsatisfiable \wrt{} an intended background theory
  (obtaining different maximal speeds is inconsistent when expecting
  \texttt{maxSp} to be functional in its fourth argument).
\item Both modifiers differ in that \emph{subject to} modifies the rule to which
  it is attached, whereas \emph{despite} has a remote effect on the rule given
  as argument.
\item They permit to structure a legal text, favouring conciseness and
  modularity: In the case of \emph{despite}, the overridden, typically more
  general rule need not be aware of the overriding, typically subordinate rules.
\item Even though these modifiers appear to be mechanisms on the meta-level in
  that they reasoning about rules, they can directly be reflected on the
  object-level.
\end{itemize}


\section{Defeasible Reasoning in a Classical Logic}\label{sec:defeasible_classical}

In this section, we will describe how to give a precise semantics to the rule
modifiers, by rewriting rules, progressively eliminating the instructions
appearing in the rule annotations so that in the end, only purely logical
rules remain. Making the meaning of the modifiers explicit can therefore be
understood as a \emph{compilation} problem.  Whereas the first preprocessing
steps (\secref{sec:preprocessing}) are generic, we will discuss two variants
of conversion into logical format (\secrefs{sec:restr_precond} and
\ref{sec:restr_deriv}). We will then discuss rule inversion
(\secref{sec:rule_inversion}) which gives our approach a non-monotonic flavour
while remaining entirely in a classical setting. Rule inversion will also be
instrumental for comparing the conversion variants in \secref{sec:comparison}.

\subsection{Rule Modifiers in Classical Logic}\label{sec:rule_modifiers_in_classical_logic}

\subsubsection{Preprocessing}\label{sec:preprocessing}

Preprocessing consists of several elimination steps that are carried out in a
fixed order.

\paragraph{\textbf{``Despite''  elimination}}

As can be concluded from the previous discussion, a
$\mathtt{despite}\; r_2$ clause appearing in rule $r_1$ is equivalent to a
$\mathtt{subjectTo}\; r_1$ clause in rule $r_2$. The first rule transformation
consists in applying exhaustively the following \emph{despite elimination}
rule transformer:

\noindent
\emph{despiteElim:}\\
$
\{r_1 \{\mathtt{restrict}: \{\mathtt{despite}\; r_2\} \uplus a_1\},\;\;
r_2\{\mathtt{restrict}: a_2\}, \dots\} \longrightarrow$\\
$\{r_1 \{\mathtt{restrict}: a_1\},\;\;
r_2\{\mathtt{restrict}:  \{\mathtt{subjectTo}\; r_1\} \uplus a_2\}, \dots\}
$

\begin{example}\label{ex:rewrite_despite}\mbox{}\\
Application of this rewrite rule to the three example rules \texttt{maxSpCarWorkday},
\texttt{maxSpCarHighway} and  \texttt{maxSpSportsCar} changes them to:

\begin{lstlisting}
rule <maxSpCarWorkday>
  {restrict: {subjectTo: maxSpCarHighway}}
rule <maxSpCarHighway>
  {restrict: {subjectTo: maxSpSportsCar}}
rule <maxSpSportsCar>
  {restrict: {subjectTo: maxSpCarWorkday}}
\end{lstlisting}
Here, only the headings are shown, the bodies of the rules are
unchanged. 
\end{example}

One defect of the rule set already becomes apparent to the human reader at
this point: the circular dependency of the rules. We will however continue
with our algorithm, applying the next step which will be to rewrite the
\texttt{\{restrict: \{subjectTo: \dots\}\}} clauses.  Please note that each
rule can be \texttt{subjectTo} several other rules, each of which may have a
complex structure as a result of transformations that are applied to it.

\paragraph{\textbf{``Subject'' To elimination}}

The rule transformer \emph{subjectToElim} does the following: it splits up the
rule into two rules, (1) its source (the rule body as originally given), and
(2) its definition as the result of applying a rule transformation function to
several rules.

\begin{example}\label{ex:rewrite_subject_to}\
Before stating the rule transformer, we show its effect on rule
\texttt{maxSpCarWorkday} of \exampleref{ex:rewrite_despite}. On rewriting
with \emph{subjectToElim}, the rule is transformed into two rules:

\begin{lstlisting}
# new rule name, body of rule unchanged
rule <maxSpCarWorkday'Orig>
   {source}
   for v: Vehicle, d: Day, r: Road
   if isCar v && isWorkday d
   then maxSp v d r 90

# rule with header and without body
rule <maxSpCarWorkday>
 {derived: {apply: 
 {restrictSubjectTo maxSpCarWorkday'Orig  maxSpSportsCar}}}
\end{lstlisting}
\end{example}

We can now state the transformation (after grouping the
\texttt{subjectTo $r_2$}, \dots, \texttt{subjectTo $r_n$} into
\texttt{subjectTo $[r_2 \dots r_n]$}):

\noindent
\emph{subjectToElim:}\\
$
\{r_1 \{\mathtt{restrict}: \{\mathtt{subjectTo}\; [r_2, \dots, r_n]\}\}, \dots \} \longrightarrow$\\
$\{r_1^o \{\mathtt{source}\}, r_1 \{\mathtt{derived:}\; \{\mathtt{apply:}\; \{
\mathtt{restrictSubjectTo}\;\; r_1^o\; [r_2 \dots r_n] \}\}\}, \dots \}
$

\paragraph{\textbf{Computation of derived rule}}
The last step consists in generating the derived rules, by evaluating the
value of the rule transformer expression marked by \texttt{apply}. The rules
appearing in these expressions may themselves be defined by complex
expressions. However, direct or indirect recursion is not allowed. For
simplifying the expressions in a rule set, we compute a rule dependency order
$\prec_R$ defined by: $r \prec_R r'$ iff $r$ appears in the defining
expression of $r'$. If $\prec_R$ is not a strict partial order (in particular, if
it is not cycle-free), then evaluation fails. Otherwise, we order the rules
topologically by $\prec_R$ and evaluate the expressions starting from the
minimal elements. Obviously, the order $\prec_R$ does not prevent rules from being recursive.

\begin{example}
It is at this point that the cyclic dependence already remarked after
\exampleref{ex:rewrite_despite} will be discovered. We have:

\noindent
\texttt{maxSpCarWorkday'Orig}, \texttt{maxSpCarHighway} $\prec_R$ \texttt{maxSpCarWorkday}\\
\texttt{maxSpCarHighway'Orig}, \texttt{maxSpSportsCar} $\prec_R$ \texttt{maxSpCarHighway}\\
\texttt{maxSpSportsCar'Orig}, \texttt{maxSpCarWorkday} $\prec_R$  \texttt{maxSpSportsCar}\\
\noindent
which cannot be totally ordered.

Let us fix the problem by changing the heading of rule
\texttt{maxSpCarHighway} from \texttt{despite} to \texttt{subjectTo}:
\begin{lstlisting}
rule <maxSpCarHighway>
  {restrict: {subjectTo: maxSpCarWorkday}}
\end{lstlisting}

After rerunning \emph{despiteElim} and \emph{subjectToElim}, we can now order
the rules:


\noindent
\{ \texttt{maxSpSportsCar'Orig}
\texttt{maxSpCarHighway'Orig},
\texttt{maxSpCarWorkday} \} $\prec_R$
\texttt{maxSpSportsCar} $\prec_R$
\texttt{maxSpCarHighway}\\
and will use this order for rule elaboration.
\end{example}

\subsubsection{Restriction via Preconditions}\label{sec:restr_precond}

Here, we propose one possible implementation of the rule transformer
\texttt{restrictSubjectTo} introduced in \secref{sec:preprocessing} that takes
a rule $r_1$ and a list of rules $[r_2 \dots r_n]$ and produces a new rule, by
adding the negation of the preconditions of $[r_2 \dots r_n]$ to $r_1$. More
formally:

\begin{itemize}
\item $\mathtt{restrictSubjectTo}\; r_1\; [] = r_1$
\item $\mathtt{restrictSubjectTo}\; r_1\; (r' \uplus rs) =$\\
  $\mathtt{restrictSubjectTo}\; (r_1(precond := precond(r_1) \AND \NOT precond(r')))\; rs$
\end{itemize}
where $precond(r)$ selects the precondition of rule $r$ and $r(precond:=p)$
updates the precondition of rule $r$ with $p$.

There is one proviso to the application of \texttt{restrictSubjectTo}: the
rules have to have the same \emph{parameter interface}: the number and types
of the parameters in the rules' \texttt{for} clause have to be the same.
Rules with different parameter interfaces can be adapted via the
\texttt{remap} rule transformer. The rule 

\begin{lstlisting}[frame=none,mathescape=true]
rule <r> for $x_1$:$T_1$ $\dots$ $x_n$:$T_n$ if Pre($x_1, \dots, x_n$) then Post($x_1, \dots, x_n$)
\end{lstlisting}
is remapped by \texttt{remap r [$y_1: S_1, \dots, y_m: S_m$] [$x_1 := e_1, \dots, x_n := e_n$]}
to 
\begin{lstlisting}[frame=none,mathescape=true]
rule <r> for $y_1$: $S_1$ $\dots$ $y_m$: $S_m$ if Pre($e_1, \dots, e_n$) then Post($e_1, \dots, e_n$)
\end{lstlisting}
Here, $e_1, \dots, e_n$ are expressions that have to be well-typed with types $E_1, \dots, E_n$ in
context $y_1: S_1, \dots, y_m: S_m$ (which means in particular that they may
contain the variables $y_i$) with $E_i \preceq T_i$,  (with the consequence that
the pre- and post-conditions of the new rule remain well-typed), where $\preceq$ is subtyping.

\begin{example}\mbox{}\\
We come back to the running example. When processing the rules in the order of
$\prec_R$, rule \texttt{maxSpSportsCar}, defined by
\texttt{apply: \{restrictSubjectTo maxSpSportsCar'Orig maxSpCarWorkday\}},
becomes:
\begin{lstlisting}
rule <maxSpSportsCar>
   for v: Vehicle, d: Day, r: Road
   if isSportsCar v && isHighway r &&
      not (isCar v && isWorkday d)
   then maxSp v d r 320
 \end{lstlisting}

 We can now state \texttt{maxSpCarHighway}, which has been defined by
 \texttt{apply: \{restrictSubjectTo maxSpCarHighway'Orig maxSpSportsCar\}}, as:

 \begin{lstlisting}
rule <maxSpCarHighway>
   for v: Vehicle, d: Day, r: Road
   if isCar v && isHighway r &&
      not (isSportsCar v && isHighway r &&
              not (isCar v && isWorkday d))  &&
      not (isCar v && isWorkday d))
   then maxSp v d r 130
\end{lstlisting}
\end{example}

One downside of the approach of adding negated preconditions is that the
preconditions of rules can become very complex. This effect is mitigated by
the fact that conditions in \texttt{subjectTo} and \texttt{despite} clauses
express specialisation or refinement and often permit substantial
simplifications. Thus, the precondition of \texttt{maxSpSportsCar} simplifies
to \texttt{isSportsCar v \&\& isHighway r \&\& isWorkday d} and the
precondition of \texttt{maxSpCarHighway} to
\texttt{isCar v \&\& isHighway r \&\& not (isSportsCar v \&\& isWorkday d)}.

\subsubsection{Restriction via Derivability}\label{sec:restr_deriv}

We now give an alternative reading of \texttt{restrictSubjectTo}. To
illustrate the point, let us take a look at a simple propositional example.

\begin{example}\label{ex:small_propositional} Take the definitions:
\begin{lstlisting}
rule <r1> if B1 then C1
rule <r2> {subjectTo: r1} if B2 then C2
\end{lstlisting}
\end{example}

Instead of saying: \texttt{r2} corresponds to
\texttt{\blue{if} B2 \&\& not B1 \blue{then} C2} 
as in \secref{sec:restr_precond}, we would now read it as
``if the conclusion of \texttt{r1} cannot be derived'', 
which could be written as
\texttt{\blue{if} B2 \&\& not C1 \blue{then} C2}.
The two main problems with this naive approach are the following:
\begin{itemize}
\item As mentioned in \secref{sec:resasoning_with_rules}, a \emph{subject to}
  restriction is often applied to rules with contradicting conclusions, so in
  the case that \texttt{C1} is \texttt{not C2}, the generated rule would be a
  tautology.
\item In case of the presence of a third rule
\begin{lstlisting}[frame=none]
rule <r3> if B3 then C1
\end{lstlisting}
a derivation of \texttt{C1} from \texttt{B3} would also block the application
of \texttt{r2}, and \texttt{subjectTo: r1} and \texttt{subjectTo: r1, r3}
would be indistinguishable.
\end{itemize}

We now sketch a solution for rule sets whose conclusion is always an atom (and
not a more complex formula).

\begin{enumerate}
\item In a preprocessing stage, all rules are transformed as follows:
  \begin{enumerate}
  \item We assume the existence of classes \texttt{Rulename$_P$}, one for each
    transformable predicate $P$ (see below).
  \item All the predicates $P$
    occurring in the conclusions of rules (called \emph{transformable
      predicates}) are converted into predicates $P^+$ with one additional
    argument of type \texttt{Rulename$_P$}. In the
    example, \texttt{C1$^+$: Rulename$_{C1}$ -> Boolean} and similarly for \texttt{C2}.
  \item The transformable predicates $P$ in conclusions of rules receive one
    more argument, which is the name \emph{rn} of the rule: $P$ is transformed
    into $P^+\; rn$. The informal reading is ``the predicate is derivable with
    rule \emph{rn}''.
  \item All transformable predicates in the preconditions of the rules receive
    one more argument, which is a universally quantified variable of type
    \texttt{Rulename$_P$} of the appropriate type, bound in the
    \texttt{for}-list of the rule.
  \end{enumerate}
\item In the main processing stage, \texttt{restrictSubjectTo} in the rule
  annotations generates rules according to:
  \begin{itemize}
\item $\mathtt{restrictSubjectTo}\; r_1\; [] = r_1$
\item $\mathtt{restrictSubjectTo}\; r_1\; (r' \uplus rs) =$\\
  $\mathtt{restrictSubjectTo}\; (r_1(precond := precond(r_1) \AND \NOT postcond(r')))\; rs$
  Thus, the essential difference \wrt{} the definition of
  \secref{sec:restr_precond} is that we add the negated post-condition and not
  the negated pre-condition.
\end{itemize}
\end{enumerate}

\begin{example} The rules of \exampleref{ex:small_propositional} are now
  transformed to:
\begin{lstlisting}[mathescape=true]
rule <r1> for rn:Rulename$_{B1}$ if B1$^+$ rn then C1$^+$ r1
rule <r2> for rn:Rulename$_{B2}$ if B2$^+$ rn and not C1$^+$ r1 then C2$^+$ r2
\end{lstlisting}
The derivability of another instance of \texttt{C1}, such as \texttt{C1$^+$ r3},
would not inhibit the application of \texttt{r2} any more.
\end{example}

\begin{example} The two rules of the running example become, after resolution
  of the \texttt{restrictSubjectTo} clauses:
\begin{lstlisting}[mathescape=true]
rule <maxSpSportsCar>
   for v: Vehicle, d: Day, r: Road
   if isSportsCar v && isHighway r &&
      not maxSp$^+$ maxSpCarWorkday v d r 90
   then maxSp$^+$ maxSpSportsCar v d r 320
rule <maxSpCarHighway>
   for v: Vehicle, d: Day, r: Road
   if isCar v && isHighway r &&
      not maxSp$^+$ maxSpCarWorkday v d r 90 &&
      not maxSp$^+$ maxSpSportsCar v d r 320
   then maxSp$^+$ maxSpCarHighway v d r 130
\end{lstlisting}
\end{example}

\subsection{Rule Inversion}\label{sec:rule_inversion}

The purpose of this section is to derive formulas that, for a given rule set,
simulate negation as failure, but are coded in a classical first-order logic,
do not require a dedicated proof engine (such as Prolog) and can be checked
with a SAT or SMT solver. The net effect is similar to the completion
introduced by Clark \cite{clark_NegAsFailure_1978}; however, the justification
is not operational as in \cite{clark_NegAsFailure_1978}, but takes inductive
closure as a point of departure. Some of the ideas are reminiscent of least
fixpoint semantics of logic programs, as discussed in
\cite{falaschi_etal_declarative_logic_langauges_1989,fages_consistency_clark_completion_1994}.
The discussion below applies to a considerably wider class of formulas.

In the following, we assume that our rules have an atomic predicate $P$ as
conclusion, whereas the precondition $Pre$ can be an arbitrarily complex
formula.  We furthermore assume that rules are in \emph{normalized form}: $P$
may only be applied to $n$ distinct variables $x_1, \dots, x_n$, where $n$ is
the arity of $P$, and the rule quantifies over exactly these variables.
For notational simplicity, we write normalized rules in logical format,
ignoring types:
$\forall x_1, \dots, x_n. Pre(x_1, \dots x_n) \IMPL Post(x_1, \dots, x_n)$.

Every rule can be written in normalized form, by applying the following
algorithm:
\begin{itemize}
\item Remove expressions or duplicate variables in the conclusion, by using
  the equivalences $P(\dots e \dots) = (\forall x. x = e \IMPL P(\dots x
  \dots))$ for a fresh variable x, and similarly $P(\dots y \dots y \dots) =
  (\forall x. x = y   \IMPL P(\dots x \dots y \dots))$.
\item Remove variables from the universal quantifier prefix if they do not
  occur in the conclusion, by using the equivalence
  $(\forall x. Pre(\dots x \dots) \IMPL P) = (\exists x. Pre(\dots x \dots))
  \IMPL P$.
\end{itemize}

For any rule set $\cal R$ and predicate $P$, we can form the set of
$P$-rules, ${\cal R}[P]$, as
\begin{align*}
\{ & \forall x_1, \dots, x_n. Pre_1[P](x_1, \dots x_n) \IMPL P(x_1, \dots, x_n),
     \dots,\\
  & \forall x_1, \dots, x_n.Pre_k[P](x_1, \dots x_n)\IMPL P(x_1, \dots,
x_n)\}
\end{align*}
as the subset of $\cal R$ containing all rules having $P$ as
post-condition. The notation $F[P]$ is meant to indicate that the $F$ can
contain $P$. It can also be taken as a \emph{functional}, \ie{} a higher-order
function having $P$ as parameter.

We say that a functional $F$ is \emph{semantically monotonic} if
\[
  (\forall x_1, \dots, x_n.\; P(x_1, \dots, x_n) \IMPL P'(x_1, \dots, x_n)) \IMPL 
  (\forall \overrightarrow{v}. F[P] \IMPL F[P'])
\]
A sufficient condition for semantic monotonicity is syntactic monotonicity:
$P$ does not occur under an odd number of negations in $F$. 

The inductive closure of a set of $P$-rules is the predicate $P^*$ defined by
the second-order formula
\[  P^*(x_1, \dots, x_n) = \forall P.\; (\bigwedge{\cal R}[P]) \IMPL P(x_1, \dots x_n) \]
where $\bigwedge{\cal R}[P]$ is the conjunction of all the rules in ${\cal R}[P]$.

$P^*$ can be understood as the least predicate satisfying the set of $P$-rules
and is the predicate that represents ``all that is known about $P$ and
assuming nothing else about $P$ is true'', and corresponds to the notion of
exhaustiveness prevalent in law texts. It can also be understood as the static
equivalent of the operational concept of negation as failure for predicate
$P$.  By the Knaster-Tarski theorem, $P^*$, as the least fixpoint of a
monotonic functional, is consistent (see a counterexample in
\exampleref{ex:syntactically_non_monotonic_rule}). 

Obviously, a second-order formula such as the definition of $P^*$ is unwieldy
in fully automated theorem proving, so we derive one particular consequence:

\begin{lemma}\label{lemma:p_star}
$P^*(x_1, \dots, x_n) \IMPL Pre_1[P^*](x_1, \dots x_n) \OR \dots \OR Pre_k[P^*](x_1, \dots x_n)$
\end{lemma}
As a consequence of the Löwenheim–Skolem theorem, there is no first-order
equivalent of $P^*$: a formula of the form $P^*$ can characterize the natural
numbers up to isomorphism, but no first-order formula can.

In the absence of such a first-order equivalent, we define the formula $Inv_P$
\[
\forall x_1, \dots, x_n.\;  P(x_1, \dots, x_n) \IMPL Pre_1(x_1, \dots x_n) \OR \dots \OR Pre_k[P](x_1, \dots x_n)
\]
called the \emph{inversion formula of  $P$}, and take it as an approximation of the
effect of $P^*$ in \lemmaref{lemma:p_star}.

As usual, a disjunction over an empty set is taken to be the falsum
$\bot$. Assume there are no defining rules for a predicate $P$, then $Inv_P =
P \IMPL \bot = \NOT P$, which corresponds to a closed-world assumption for $P$.

\begin{example}\label{ex:syntactically_non_monotonic_rule}
  One motivation for the monotonicity constraint is the following: The
  simplest example of a rule that is not syntactically monotonic is
  $\NOT P \IMPL P$. Its inversion is $P \IMPL \NOT P$. The two formulas
  together, $P \IFF \NOT P$, are inconsistent.
\end{example}

Inversion formulas can be automatically derived and added to the rule set in
L4 proofs; they turn out to be essential for consistency properties. For
example, the functionality of \texttt{maxSp} stated in \figref{fig:assertions}
is not provable without the inversion formula of \texttt{maxSp}.

To avoid misunderstandings, we should emphasize that this approach is entirely
based on a classical monotonic logic, in spite of non-monotonic
effects. Adding a new $P$-rule may invalidate previously provable facts, but
this is only so because the new rule alters the inversion formula of $P$.

\subsection{Comparison}\label{sec:comparison}

One may wonder whether, starting from the same set of rules, the transformations in
\secref{sec:restr_precond} and \secref{sec:restr_deriv} produce
equivalent rules. On the face of it, this is not so, because the
transformation via derivability modifies the arity of the predicates, so the
rule sets have different models.

We will however show that the two rule sets have corresponding sets of
models. This will be made more precise in the following. To fix notation,
assume ${\cal R}_M$ to be a set of rules annotated with rule modifiers. Let
${\cal R}_P$ be the set of rules obtained from ${\cal R}_M$ through the rule
translation via preconditions of \secref{sec:restr_precond}, and similarly
${\cal R}_D$ the set of rules obtained from ${\cal R}_M$ through the rule
translation via derivability of \secref{sec:restr_deriv}. From these rule
sets, we obtain formula sets ${\cal F}_P$ respectively ${\cal F}_D$ by
\begin{itemize}
\item translating rules to formulas;
\item adding inversion formulas $Inv_C$ for all
  the transformable predicates $C$ of the rule set;
\end{itemize}

\begin{lemma}\label{lemma:mp_to_md}
  Any model ${\cal M}_P$ of ${\cal F}_P$ can be transformed into a model
  ${\cal M}_D$ of ${\cal F}_D$.
\end{lemma}

The proof will be given in \secref{sec:comparison_proofs} after  \lemmaref{lemma:mp_to_md_with_proof}.

\begin{lemma}\label{lemma:md_to_mp}
  Any model ${\cal M}_D$ of ${\cal F}_D$ can be transformed into a model
  ${\cal M}_P$ of ${\cal F}_P$.
\end{lemma}

The proof will be given in \secref{sec:comparison_proofs} after
\lemmaref{lemma:md_to_mp_with_proof}

We should emphasize that, in the proof of \lemmaref{lemma:md_to_mp}, the
inversion formulas play a decisive role.


\section{Defeasible Reasoning with Answer Set Programming}\label{sec:defeasible_asp}

\subsection{Introduction}
The purpose of this section is to give an account of the work we have been doing using Answer Set Programming (ASP) to formalize and reason about legal rules. This approach is complementary to the one described before using SMT solvers. Here we will not go too much into the details of how various L4 language constructs map to the ASP formalisation. Our intention, rather, is to present how some core legal reasoning tasks can be implemented in ASP while keeping the ASP representation readable and intuitive and respecting the idea of having an `isomorphism' between the rules and the encoding. Going forward, our intention is to develop a method to compile L4 code to a suitable ASP representation like the one we shall now present. We formalize the notion of what it means to `satisfy' a rule set. We will do this in a way that is most amenable to ASP. Please see the appendix for a brief overview of ASP and references for further reading.\\

Our work in this section is inspired by \cite{DBLP:conf/iclp/WanGKFL09}. Readers will note that there are similarities between the use of predicates such as $according\_to$, $defeated$, $opposes$ in our ASP encoding, to reason about rules interacting with each other, and similar predicates that the authors of \cite{DBLP:conf/iclp/WanGKFL09} use in their work. However our ASP implementation is much more specific to legal reasoning whereas they seek to implement very general logic based reasoning mechanisms. We independently developed our `meta theory' for how rule modifiers interact with the rules and with each other and there are further original contributions like a proposed axiom system for what we call `legal models'. An interesting avenue of future work could be to compare our approaches within the framework of legal reasoning.

The work in this section builds on the work in \cite{morris21:_const_answer_set_progr_tool} and hence uses some of the same notation and terminology. The author of \cite{morris21:_const_answer_set_progr_tool} was a member of the same research group as the authors of this paper at SMU in 2020--2021.

\subsection{Formal Setup}
Let the tuple $Config = (R,F,M,I)$ denote a $configuration$ of legal rules. The set $R$ denotes a set of rules of the form $pre\_con(r)\rightarrow concl(r)$. These are `naive' rules with no information pertaining to any of the other rules in $R$. $F$ is a set of positive atoms that describe facts of the legal scenario we wish to consider. $M$ is a set of the binary predicates $despite$, $subject\_to$ and $strong\_subject\_to$. $I$ is a collection of minimal inconsistent sets of positive atoms. Henceforth for a rule $r$, we may write $C_{r}$ for its conclusion $Concl(r)$.

Note that, throughout this section, given any rule $r$, $C_{r}$ is assumed to be a single positive atom. That is, there are no disjunctions or conjunctions in rule conclusions. Also any rule pre-condition ($pre\_con(r)$) is assumed to be a conjunction of positive and negated atoms. Here negation denotes `negation as failure'.  

Throughout this document, whenever we use an uppercase or lowercase letter (like $r$, $r_{1}$, $R$ etc.) to denote a rule that is an argument, in a binary predicate, we mean the unique integer rule \textsc{id} associated with that rule. The binary predicate $legally\_valid(r,c)$ intuitively means that the rule $r$ is `in force' and it has conclusion $c$. Here $r$ typically is an integer referring to the rule \textsc{id} and $c$ is the atomic conclusion of the rule. The unary predicate $is\_legal(c)$ intuitively means that the atom $c$ legally holds/has legal status. The predicates $despite$, $subject\_to$ and $strong\_subject\_to$ all cause some rules to override others. Their precise properties will be given next.

\subsection{Semantics}

A set $S$ of $is\_legal$ and $legally\_valid$ predicates is called a \textit{legal model} of $Config = (R,F,M,I)$, if and only if
\begin{description}
\item[(A1)]$\forall f \in F$ $is\_legal(f) \in S$.

\item[(A2)] $\forall r \in R$, if $legally\_valid(r,C_{r}) \in S$. then $S\models is\_legal(pre\_con(r))$ and $S\models is\_legal(C_{r})$ \footnote{By $S\models is\_legal(pre\_con(r))$ we mean that for each positive atom $b_{i}$ in the conjunction, $is\_legal(b_{i}) \in S$ and for each negated body atom $not$ $b_{j}$ in the conjunction $is\_legal(b_{j})\notin S$ }

\item[(A3)] $\forall c$, if $is\_legal(c) \in S$, then either $c\in F$ or there exists $r \in R$ such that $legally\_valid(r,C_{r}) \in S$ and $c= C_{r}$.

\item[(A4)] $\forall r_{i}, r_{j} \in R$, if $despite(r_{i}, r_{j}) \in M$ and $S\models is\_legal(pre\_con(r_{j}))$, then $legally\_valid(r_{i},C_{r_{i}}) \notin S$

\item[(A5)] $\forall r_{i}, r_{j} \in R$, if $strong\_subject\_to(r_{i}, r_{j}) \in M$ and $legally\_valid(r_{i},C_{r_{i}}) \in S$, then $legally\_valid(r_{j},C_{r_{j}}) \notin S$

\item[(A6)] $\forall r_{i},r_{j} \in R$ if $subject\_to(r_{i},r_{j}) \in M$, and $legally\_valid(r_{i},C_{r_{i}}) \in S$ and there exists a minimal conflicting set $k \in I$ such that $C_{r_{i}} \in k$ and $C_{r_{j}}\in k$ and $is\_legal(k\setminus \{C_{r_{j}})\})\subseteq S $, then $legally\_valid(r_{j},C_{r_{j}}) \notin S$. Note than in our system, any minimal inconsistent set must contain at least 2 atoms. \footnote{For a set of atoms $A$, by $is\_legal(A)$, we mean the set $\{is\_legal(a)\mid a\in A\}$} 

\item[(A7)] $\forall r\in R$, if $S\models pre\_con(r)$, but $legally\_valid(r,C_{r})\notin S$, then it must be the case that at least one of A4 or A5 or A6 has caused the exclusion of $legally\_valid(r,C_{r})$. That is if $S\models pre\_con(r)$, then unless this would violate one of A5, A6 or A7, it must be the case that $legally\_valid(r,C_{r})\in S$.
\end{description}

\subsection{Some remarks on axioms A1--A7}
We now give some informal intuition behind some of the axioms and their intended effects.

A1 says that all facts in $F$ automatically gain legal status, that is, they legally hold. The set $F$ represents indisputable facts about the legal scenario we are considering.

A2 says that if a rule is `in force' then it must be the case that both the pre-condition and conclusion of the rule have legal status. Note that it is not enough if simply require that the conclusion has legal status as more than one rule may enforce the same conclusion or the conclusion may be a fact, so we want to know exactly which rules are in force as well as their conclusions.

A3 says that anything that has legal status must either be a fact or be a conclusion of some rule that is in force.

A4--A6 describe the semantics of the three modifiers. The intuition for the three modifiers will be discussed next. Firstly, it may help the reader to read the modifiers in certain ways. $despite(r_{i},r_{j})$ should be read as `despite $r_{i}$, $r_{j}$'. Thus $r_{i}$ here is the `subordinate rule' and $r_{j}$ is the `dominating' rule. The idea here is that once the precondition of the dominating rule $r_{j}$ is true, it invalidates the subordinate rule $r_{i}$ regardless of whether the dominating rule itself is then invalidated by some other rule. For \textit{strong subject to}, the intended reading for $strong\_subject\_to(r_{i},r_{j})$ is something like `(strong) subject to $r_{i}$, $r_{j}$'. Here $r_{i}$ can be considered the dominating rule and $r_{j}$ the subordinate. Once the dominating rule is in force, then it invalidates the subordinate rule. The intended reading for $subject\_to(r_{i},r_{j})$ is `subject to $r_{i}$, $r_{j}$'. For the subordinate rule $r_{j}$ to be invalidated, it has to be the case that the dominating rule $r_{i}$ is in force and there is a minimal inconsistent set $k$ in $I$ that contains the two atoms in the conclusions of the two rules and, $i\setminus\{C_{r_{j}}\}\subseteq S$. These minimal inconsistent sets along with the \textit{subject to} modifier give us a way to incorporate a classical-negation-like effect into our system. We are able to say which things contradict each other. Note that in our system, if say $\{a,b\}$ is a minimal inconsistent set, then it is possible for both $is\_legal(a)$ and $is\_legal(b)$ to be in a single legal model, if they are both facts or they are conclusions of rules that have no modifiers linking them. These minimal inconsistent sets only play a role where a $subject\_to$ modifier is involved. The reason for doing this is that this offers greater flexibility rather than treating $a$ and $b$ as pure logical negatives of each other that cannot be simultaneously true in a legal model. We will give examples later on to illustrate these modifiers.

A7 says essentially that A4--A6 represent the only ways in which a rule whose pre-condition is true may nevertheless be invalidated, and any rule whose precondition is satisfied and is not invalidated directly by some instance of A4--A6, must be in force.

\subsection{Non-existence of legal models}
Note that there may be configurations for which no legal models exist. This is most easily seen in the case where there is only one rule, the pre-condition of that rule is given as fact, and the rule is strongly subject to itself. See the appendix for some further examples of `pathological' rule configurations.

\subsection{ASP encoding}
Here is an ASP encoding scheme given a configuration $Config = (R,F,M,I)$ of legal rules.
\begin{lstlisting}[language=Prolog, numbers=left]
% For any f in F, we have:
is_legal(f). 

% All the modifiers get added as facts like for example:
despite(1,2).

% Any rule r in R is encoded using the general schema:
according_to(r,C_r):-is_legal(pre_con(r)).

% Given a minimal inconsistent set {a_1,a_2,...,a_n}, this corresponds to a set of rules:
opposes(a_1,a_2):-is_legal(a_2),is_legal(a_3),...,is_legal(a_n).
opposes(a_1,a_3):-is_legal(a_2),is_legal(a_4)...,is_legal(a_n).
               .
               .
               .
opposes(a_n-1,a_n):-is_legal(a_1),...,is_legal(a_n-2).               

% Opposes is a symmetric relation
opposes(X,Y):-opposes(Y,X).


% Encoding for 'despite'
defeated(R,C,R1) :-
    according_to(R,C), according_to(R1,C1), despite(R,R1).

%Encoding for 'subject_to'
defeated(R,C,R1) :-
    according_to(R,C), legally_valid(R1,C1),
    opposes(C,C1), subject_to(R1,R).

% Encoding for 'strong_subject_to'
defeated(R,C,R1) :-
    according_to(R,C), legally_valid(R1,C1),
    strong_subject_to(R1,R).

not_legally_valid(R) :- defeated(R,C,R1).

legally_valid(R,C):-according_to(R,C),not not_legally_valid(R).

is_legal(C):-legally_valid(R,C).
\end{lstlisting}
\subsection{Lemma}

\begin{lemma}\label{lemma:legal_model_of_config}
For a configuration $Config=(R,F,M,I)$, let the above encoding be the program $ASP_{Config}$. Then given an answer set $A_{Config}$ of $ASP_{Config}$ let $S_{A_{Config}}$ be the set of $is\_legal$ and $legally\_valid$ predicates in $A_{Config}$. Then $S_{A_{Config}}$ is a legal model of $Config$. 
\end{lemma}

$Proof$ See Appendix. $\square$

\subsection{Example}
Let us now give an example to illustrate the various concepts~/~modifiers discussed above.
Consider 4 basic rules:
\begin{enumerate}
  \item If Bob is wealthy, he must buy a Rolls-Royce.
  \item If Bob is wealthy, he must buy a Mercedes.
  \item If Bob is wealthy, he may spend up to 2 million dollars on cars.
  \item If Bob is extremely wealthy, he may spend up to 10 million dollars on cars.
\end{enumerate}
Suppose we know that the Rolls-Royce and Mercedes together cost more
than 2 million but each is individually less than 2 million. We also
have that rules 1 and 2 are each subject to rule 3 and despite rule 1,
rule 4 holds. Additionally, we also have the fact that Bob is
wealthy. In this situation we would expect 2 legal models. One in which
exactly rule 1 and rule 3 are legally valid and one in which exactly
rule 2 and rule 3 are legally valid. Let us see what our encoding
looks like.
\begin{lstlisting}[language=Prolog, numbers=left]
is_legal(wealthy(bob)).
% Rules
according_to(1,must_buy(rolls,bob)) :- is_legal(wealthy(bob)).
according_to(2,must_buy(merc,bob)) :- is_legal(wealthy(bob)).
according_to(3,may_spend_up_to_one_mill(bob)) :-
    is_legal(wealthy(bob)).
according_to(4,may_spend_up_to_ten_mill(bob)) :-
    is_legal(extremely_wealthy(bob)).

% {(must_buy(rolls,bob),must_buy(merc,bob), may_spend_up_to_one_mill(bob)} is a min. inconsistent set.

opposes(must_buy(rolls,bob),must_buy(merc,bob)) :-
    is_legal(may_spend_up_to_one_mill(bob)).

opposes(must_buy(rolls,bob),may_spend_up_to_one_mill(bob)) :-
    is_legal(must_buy(merc,bob)).

opposes(must_buy(merc,bob),may_spend_up_to_one_mill(bob)) :-
    is_legal(must_buy(rolls,bob)).

opposes(X,Y):-opposes(Y,X).

subject_to(3,1).
subject_to(3,2).
despite(3,4).

% Encoding for 'despite'
defeated(R,C,R1) :-
    according_to(R,C),according_to(R1,C1),despite(R,R1).

% Encoding for 'subject_to'
defeated(R,C,R1) :-
    according_to(R,C), legally_valid(R1,C1),
    opposes(C,C1), subject_to(R1,R).

% Encoding for 'strong_subject_to'
defeated(R,C,R1) :-
    according_to(R,C), legally_valid(R1,C1),
    strong_subject_to(R1,R).

not_legally_valid(R) :- defeated(R,C,R1).

legally_valid(R,C) :- 
    according_to(R,C), not not_legally_valid(R).

is_legal(C):-legally_valid(R,C).
\end{lstlisting}
Running the the program gives exactly 2 answer sets corresponding to the legal models described above. Now if we add say $strong\_subject\_to(3,1)$ to the set of modifiers then we get exactly one legal model/answer set where exactly rule 3 and rule 2 are legally valid but not rule 1 because it has been invalidated due to rule 3 being legally valid with no regard for the minimal inconsistent sets.

Lastly, if we add $extremely\_wealthy(bob)$ to the set of facts, then we get a single legal model/answer set where exactly rule 1, rule 2 and rule 4 are legally valid. This is because the rule 3 has been invalidated and hence there are no constraints now on the validity of rule 1 and rule 2. See appendix for further remarks on this example.


\section{Conclusions}\label{sec:conclusions}

This paper has discussed different approaches for representing defeasibility
as used in law texts, by annotating rules with modifiers that explicate their
relation to other rules. We have notably presented two encodings in classical
logic (\secrefs{sec:restr_precond} and \ref{sec:restr_deriv}) and explained
how they are related (\secref{sec:comparison}). Quite a different approach,
based on Answer Set Programming, is presented in
\secref{sec:defeasible_asp}.

An experimental implementation of the L4 ecosystem is under
way\footnote{\url{https://github.com/smucclaw/baby-l4}}, but it has not yet
reached a stable, user-friendly status. It is implemented in Haskell and features
an IDE based on VS Code and natural language processing via Grammatical
Framework\cite{ranta_grammatical_2004}. Currently, only the coding of
\secrefs{sec:restr_precond} has been implemented---the coding of
\secref{sec:restr_deriv} would require a much more profound transformation of
rules and assertions. Likewise, the ASP coding of \secref{sec:defeasible_asp}
has only been done manually. The interaction with SMT solvers is done through
an SMT-LIB \cite{BarFT_SMTLIB} interface, thus opening the possibility to
interact with a wide range of solvers. As our rules typically contain
quantification, reasoning with quantifiers is crucial, and the best support
currently seems to be provided by Z3 \cite{demoura_bjorner_z3_2008}.

We are still at the beginning of the journey. A theoretical comparison of the
classical and ASP approaches presented here still has to be carried out, and
it has to be propped up by an empirical evaluation. For this purpose, we are
currently in the process of coding some real-life law texts in L4. We are
fully aware of shortcomings of the current L4, which will strongly evolve in
the next months, to include reasoning about deontics and about temporal
relations. Integrating these aspects will not be easy, and this is one reason
for not committing prematurely to one particular logical framework.


\paragraph{Acknowledgements.}
The contributions of the members of the L4 team to this common effort are
thankfully acknowledged, in particular of Jason Morris who contributed his
experience with Answer Set Programming; of Jacob Tan and Ruslan Khafizov who
have participated in discussions about its contents and commented on the
paper; and of Liyana Muthalib who has proof-read a previous version.

This research is supported by the National Research Foundation (NRF),
Singapore, under its Industry Alignment Fund –- Pre-Positioning Programme, as
the Research Programme in Computational Law. Any opinions, findings and
conclusions or recommendations expressed in this material are those of the
authors and do not reflect the views of National Research Foundation,
Singapore.

\newpage
\bibliographystyle{splncs04}
\bibliography{main}

\newpage
\appendix
\section{An Overview of the L4 Language: Details}\label{sec:l4_language_app}

Let us give some more details about the L4 language: its class and type
definition mechanism, and the way it handles proof obligations.

\subsection{Terminology and Class Definitions}\label{sec:classdefs}

The definition in \figref{fig:classdefs} introduces classes for vehicles, days
and roads. 

\begin{figure}[h!]
\begin{lstlisting}
class Vehicle {
   weight: Integer
}
class Car extends Vehicle {
   doors: Integer
}
class Truck extends Vehicle
class SportsCar extends Car

class Day
class Workday extends Day
class Holiday extends Day

class Road
class Highway extends Road
\end{lstlisting}
  \caption{Class definitions of speedlimit example}\label{fig:classdefs}
\end{figure}

Classes are arranged in a tree-shaped hierarchy, having a class named
\texttt{Class} as its top element. Classes that are not explicitly derived
from another class via \texttt{extends} are implicitly derived from
\texttt{Class}. A class $S$ derived from a class $C$ by \texttt{extends} will
be called a subclass of $C$, and the immediate subclasses of \texttt{Class}
will be called \emph{sorts} in the following. Intuitively, classes are meant
to be sets of entitities, with subclasses being interpreted as
subsets. Different subclasses of a class are not meant to be disjoint.

Class definitions can come with attributes, in braces. These attributes can be
of simple type, as in the given example, or of higher type (the notion of type
will be explained in \secref {sec:fundecls}). In a declarative reading,
attributes can be understood as a shorthand for function declarations that
have the class they are defined in as additional domain. Thus, the attribute
\texttt{weight} corresponds to a top-level declaration \texttt{weight: Vehicle
  -> Integer}. In a more operational reading, L4 classes can be understood as
prototypes of classes in object-oriented programming languages, and an
alternative field selection syntax can be used: For \texttt{v: Vehicle}, the
expression \texttt{v.weight} is equivalent to \texttt{weight(v)}, at least
logically, even though the operational interpretations may differ.

\subsection{Types and Function Declarations}\label{sec:fundecls}

L4 is an \emph{explicitly} and \emph{strongly typed} language: all entities
such as functions, predicates and variables have to be declared before being
used. One purpose of this measure is to ensure that the executable sublanguage
of L4, based on the simply-typed lambda calculus with subtyping, enjoys a type
soundness property: evaluation of a function cannot produce a dynamic type
error.

\figref{fig:fundecls} shows two function declarations. Functions with
\texttt{Boolean} result type will sometimes be called \emph{predicates} in the
following, even though there is no syntactic difference. All the declared
classes are considered as elementary types, as well as \texttt{Integer},
\texttt{Float}, \texttt{String} and \texttt{Boolean} (which are internally also treated as
classes). If $T_1, T_2, \dots T_n$ are types, then so are function types
\texttt{$T_1$ -> $T_2$} and tuple types \texttt{($T_1$, $\dots$ ,$T_n$)}. The
type system and the expression language, to be presented later, are
higher-order, but extraction to some solvers will be limited to
(restricted) first-order theories.

\begin{figure}[h]
\begin{lstlisting}
decl isCar : Vehicle -> Boolean
decl maxSp : Vehicle -> Day -> Road -> Integer -> Boolean
\end{lstlisting}
  \caption{Declarations of speedlimit example}\label{fig:fundecls}
\end{figure}

The nexus between the terminological and the logical level is established with
the aid of \emph{characteristic predicates}. Each class $C$ which is a
subclass of sort $S$ gives rise to a declaration \texttt{is$C$: $S$ ->
  Boolean}. An example is the declaration of \texttt{isCar} in
\figref{fig:fundecls}. In the L4 system, this declaration, as well as the
corresponding class inclusion axiom, are generated
automatically.

Two classes derived from the same base class (thus: \texttt{$C_1$ extends $B$}
and \texttt{$C_2$ extends $B$}) are not necessarily disjoint. 

From the subclass relation, a \emph{subtype} relation $\preceq$ can be defined
inductively as follows: if \texttt{$C$ extends $B$}, then $C \preceq B$, and
for types $T_1, \dots, T_n, T_1', \dots, T_n'$,
if $T_1 \preceq T_1', \dots, T_n \preceq T_n'$, 
then \texttt{$T_1'$ -> $T_2 \; \preceq \; T_1$ -> $T_2'$} 
and \texttt{($T_1$, $\dots$, $T_n$) $\preceq$ ($T_1'$, $\dots$, $T_n'$)}.

Without going into details of the type system, let us remark that it has been
designed to be compatible with subtyping: if an element of a type is
acceptable in a given context, then so is an element of a subtype. In
particular,
\begin{itemize}
\item for field selection, if $C'$ is a class having field $f$ of type $T$,
  and $C \preceq C'$, and $c : C$, then field selection is well-typed with $c.f : T$.
\item for function application, if $f: A' \mbox{\texttt{->}} B$ and $a:A$ and
  $A \preceq A'$, then function application is well-typed with $f\; a : B$.
\end{itemize}

\subsection{Assertions}\label{sec:assertions}

Assertions are statements that the L4 system is meant to verify or to
reject -- differently said, they are proof obligations. These assertions are verified
relative to a rule set comprising some or all of the rules and facts stated
before.

The active rule set used for verification can be configured, by adding rules
to or deleting rules from a default set. Assume the active rule set consists
of $n$ rules whose logical representation is $R_1 \dots R_n$, and assume the
formula of the assertion is $A$. The proof obligation can then be checked for
\begin{itemize}
\item  \emph{satisfiability}: in this case, $R_1 \AND \dots \AND R_n \AND A$
  is checked for satisfiability.
\item \emph{validity}: in this case, $R_1 \AND \dots \AND R_n \IMPL A$ is
  checked for validity.
\end{itemize}
In either case, if the proof fails, a model resp.{} countermodel is produced.
In the given example, the SMT solver checks the validity of the formula and
indeed returns a countermodel that leads to contradictory prescriptions of the
maximal speed: if the vehicle is a car, the day a workday and the road a
highway, the maximal speed can be 90 or 130, depending on the rule applied.

The assertion \texttt{maxSpFunctional} of \figref{fig:assertions} can be considered an essential
consistency requirement and a rule system violating it is inconsistent \wrt{} the intended semantics of \texttt{maxSp}. One
remedial action is to declare one of the rules as default and the other rule
as overrriding it.

After this repair action, \texttt{maxSpFunctional} will be provable (under
additional natural conditions described in \secref{sec:rule_inversion}). We can now
continue to probe other consistency requirements, such as exhaustiveness
stating that a maximal speed is defined for every combination of vehicle:

\begin{lstlisting}
assert <maxSpExhaustive>
   exists sp: Integer. maxSp instVeh instDay instRoad sp
\end{lstlisting}

The intended usage scenario of L4 is that by an interplay of proving
assertions and repairing broken rules, one arrives at a rule set satisfying
general principles of coherence, completeness and other, more elusive
properties such as fairness (at most temporary exclusion from essential
resources or rights).


\section{Reasoning with and about Rules - Motivation}\label{sec:resasoning_with_rules_app}

To illustrate the use of rule modifiers discussed in
\secref{sec:resasoning_with_rules}, we consider a realistic law text,
Singapore's Professional Conduct Rules \S~34
\cite{professional_conduct_rules}. This case study has been investigated in
detail in \cite{morris21:_const_answer_set_progr_tool}. Here is an excerpt of
the rules:

\begin{description}
\item[(1)] A legal practitioner must not accept any executive appointment
  associated with any of the following businesses: 
  \begin{description}
  \item[(a)] any business which detracts from, is incompatible with, or
    derogates from the dignity of, the legal profession;
  \item[(b)] any business which materially interferes with the legal
    practitioner’s primary occupation of practising as a lawyer; (\dots)
  \end{description}
\item[(5)] Despite paragraph (1)(b), but subject to paragraph (1)(a) and (c)
  to (f), a locum solicitor may accept an executive appointment in a business
  entity which does not provide any legal services or law-related services, if
  all of the conditions set out in the Second Schedule are satisfied.
\end{description}

The two main notions developed in the Conduct Rules are which executive appointments a legal
practictioner \emph{may} or \emph{must not} accept under which
circumstances. As there is currently no direct support for deontic logics in
L4, these notions are defined as two predicates \texttt{MayAccept} and
\texttt{MustNotAccept}, with the intended meaning that these two notions are
contradictory, and this is indeed what will be provable after a complete
formalization.

Let us here concentrate on the modifiers \emph{despite} and \emph{subject
  to}. A synonym of ``despite'' that is often used in legal texts is
``notwithstanding'',  and a synonym of
``subject to'' is ``except as provided in'', see \cite{adams_contract_drafting_2004}.

The reading of rule (5) is the following:
\begin{itemize}
\item ``subject to paragraph (1)(a) and (c) to (f)'' means: rule (5) applies
  as far as (1)(a) and (c) to (f) is not established. Differently said, rules
  (1)(a) and (c) to (f) undercut or defeat rule (5).

  One way of explicitating the ``subject to'' clause would be to rewrite (5)
  to: ``Despite paragraph (1)(b), provided the business does not detract from,
  is incompatible with, or derogate from the dignity of, the legal profession;
  and provided that not [clauses (1)(c) to (f)]; then a locum
  sollicitor\footnote{in our class-based terminology, a subclass of legal
    practitioner} may accept an executive appointment.''

\item ``despite paragraph (1)(b)'' expresses that rule (5) overrides rule
  (1)(b). In a similar spirit as the ``subject to'' clause, this can be made
  explicit by introducing a proviso, however not locally in  rule (5), but
  remotely in rule (1)(b).

  One way of explicitating the ``despite'' clause of rule (5) would be to
  rewrite (1)(b) to: ``Provided (5) is not applicable, a legal practitioner
  must not accept any executive appointment associated with any business which
  materially interferes with the legal practitioner’s primary occupation of
  practising as a lawyer.''
\end{itemize}

The astute reader will have remarked that the treatment in both cases is
slightly different, and this is not related to the particular semantics of
\emph{subject to} and \emph{despite}: we can state defeasibility
\begin{itemize}
\item either in the form of (negated) preconditions of rules: ``rule $r_1$ is
  applicable if the preconditions of $r_2$ do not hold'';
\item or in the form of (negated) derivability of the postcondition of rules: ``rule $r_1$ is
  applicable if the postcondition of $r_2$ does not hold''.
\end{itemize}


\section{Proofs: Comparison of  Rule Transformation Strategies}\label{sec:comparison_proofs}

We restate and give detailed proofs of two lemmas of \secref{sec:rule_inversion}.

\begin{lemma}\label{lemma:mp_to_md_with_proof}
  Any model ${\cal M}_P$ of ${\cal F}_P$ can be transformed into a model
  ${\cal M}_D$ of ${\cal F}_D$.
\end{lemma}

\begin{proof}
  We consider the transformation of a model ${\cal M}_P$ to a model
  ${\cal M}_D$, and assume ${\cal M}_P$ is a model of ${\cal F}_P$. 
  We now construct an interpretation ${\cal M}_D$ for the formulas with the
  signature over ${\cal F}_D$.

  The interpretation ${\cal M}_D$ will be the same as ${\cal M}_P$,
  except for (1) the interpretation of the new types \texttt{Rulename$_C$},
  each of which will be chosen to be the set of all rule names having $C$ as
  conclusion, and (2) the interpretation of the new predicates $C^+$ on which
  we will focus now: 
  For each rule  $\forall x_1, \dots, x_n.\; Pre(x_1,
  \dots, x_n) \IMPL C(x_1, \dots, x_n)$ with name  $rn$, whenever the $n$-tuple
  $(a_1, \dots, a_n)$ satsifies the precondition $Pre$ under ${\cal M}_P$ and, consequently,
  $(a_1, \dots, a_n) \in C^{{\cal M}_P}$, we will have   $(rn, a_1, \dots, a_n) \in (C^+)^{{\cal M}_D}$.

  It remains to be shown that ${\cal M}_D$ is indeed a model of ${\cal
    M}_D$. We show that related formulas in ${\cal F}_P$ and ${\cal F}_D$
  are interpreted as true in ${\cal M}_P$ resp.{} ${\cal M}_D$, where two
  formulas are \emph{related} if they are rules originating from the same rule
  of ${\cal R}_M$, or if they are related inversion predicates $Inv_C$ and $Inv_{C^+}$.

  We first address related rules. The proof is by well-founded induction over
  the rule order $\prec_R$. Consider a rule $r_P \in {\cal F}_P$ with rule
  name $rn_P$ which by construction has the form
  $r_p = \forall x_1, \dots x_n.\; pre_P^o \AND \NOT pre_P^1 \AND \NOT pre_P^k
  \IMPL C(x_1, \dots, x_n)$.
  We make a case distinction:
  \begin{itemize}
  \item Assume that for arguments $(a_1, \dots, a_n)$, interpretation
    ${\cal M}_P$ satisfies the precondition
    $pre_P^o \AND \NOT pre_P^1 \AND \NOT pre_P^k$ and thus also the
    conclusion. In this case, $(rn_P, a_1, \dots, a_n)\in (C^+)^{{\cal M}_D}$, thus
    satisfying the related rule $r_D \in {\cal F}_D$.
  \item Assume that for arguments $(a_1, \dots, a_n)$, interpretation
    ${\cal M}_P$ does not satisfy the precondition. Either $pre_P^o$ is not
    satisfied, leading again to a satisfying assignment of the related rule
    $r_D$, or one of the $pre_P^i$ is satisfied.

    In this case, as the rule $r_P^i$ with precondition $pre_P^i$ is strictly
    smaller than $r_P$ \wrt{} $\prec_R$, by induction hypothesis, also the
    postcondition of $r_P^i$ will be satisfied, so that in ${\cal M}_D$, one
    negated precondition of the related rule $r_D$ is not satisfied, so $r_D$
    is satisfied.
  \end{itemize}

  Once the equi-satisfiability of related rules has been established, it is
  easy to do so for related inversion predicates $Inv_C$ and $Inv_{C^+}$.
\end{proof}

\begin{lemma}\label{lemma:md_to_mp_with_proof}
  Any model ${\cal M}_D$ of ${\cal F}_D$ can be transformed into a model
  ${\cal M}_P$ of ${\cal F}_P$.
\end{lemma}

\begin{proof} (Sketch)
  In analogy to \lemmaref{lemma:mp_to_md}, we start from a model ${\cal M}_D$
  of ${\cal F}_D$ and construct a model ${\cal M}_P$ of ${\cal F}_P$. 

  As in \lemmaref{lemma:mp_to_md}, the proof is by induction on $\prec_R$.
  Consider a rule $r_D \in {\cal F}_D$ with rule
  name $rn_D$ which by construction has the form
  $r_D = \forall x_1, \dots x_n.\; pre_D^o \AND \NOT post_D^1(rn_1) \AND \NOT post_D^k(rn_k)
  \IMPL C^+(rn_D, x_1, \dots, x_n)$. Again, we make a case distinction:
  \begin{itemize}
  \item Assume that for arguments $(a_1, \dots, a_n)$, interpretation
    ${\cal M}_D$ satisfies the precondition and thus also the conclusion. In
    this case, $(a_1, \dots, a_n)\in C^{{\cal M}_P}$, thus satisfying the
    related rule $r_P \in {\cal F}_P$.
  \item Assume that for arguments $(a_1, \dots, a_n)$, interpretation
    ${\cal M}_D$ does not satisfy the precondition. The interesting situation
    is if one $post_D^i(rn_i)$ is satisfied. At this point, we need the
    inversion formula of $post_D^i$, of the form
    $\forall r.\; post_D^i(r) \IMPL P_1(r) \OR \dots \OR P_p(r)$. The rule
    name $rn_i$ permits to select precisely the precondition $P_j$ of the
    related formula
    $r_P = \forall x_1, \dots x_n.\; pre_P^o \AND \NOT pre_P^1 \AND \NOT
    pre_P^k \IMPL C(x_1, \dots, x_n)$.
  \end{itemize}
\end{proof}
\section{Brief outline of ASP}
First we shall give a brief overview of Answer Set Programming. ASP is a declarative programming language used mainly in Knowledge Representation and Reasoning to model rules, facts, integrity constraints etc. within a particular scenario that one wishes to consider. A rule in ASP has the form:
\[h\leftarrow b_{1},b_{2}..,b_{k},not\; b_{k+1}...,not\; b_{n}.\]
Here $h$ and $b_{1}$...,$b_{n}$ are atoms. For an atom $b_{i}$, $not$ $b_{i}$ is the negated atom where the $not$ represents negation as failure. Informally $not$ $b_{i}$ is true exactly when $b_{i}$ cannot be derived. This is also sometimes known as the `closed world assumption'. Intuitively the rule above says that when $b_{1},b_{2}..,b_{k},not$ $b_{k+1}...,not$ $b_{n}$ are all true, $h$ is true. $h$ is also sometimes known as the head of the rule and the positive and negated atoms $b_{1},b_{2}..,b_{k},not$ $b_{k+1}...,not$ $b_{n}$ form the body. A rule with only a head and an empty body is called a fact. A logic program is a set of facts and rules. (In fact ASP can also model other things like integrity constrains, disjunctions in rule heads etc, but we will not be using these features in our paper). When a logic program is passed to an ASP solver, the solver returns a set of $stable$ $models$ (also known as $answer$ $sets$) which make all the rules and facts in the logic program true. The set of $stable$ $models$ of a logic program is calculated using the $stable$ $model$ $semantics$ for ASP. For logic programs without negation-as-failure, the set of stable models is exactly the set of subset minimal models of the program. For logic programs with negation as failure stable models are most commonly defined using a construction known as the $reduct$ of a logic program with respect to an $Herbrand$ $interpretation$. Please see \cite{asp_background} for more details on ASP and the stable model semantics.

\section{ASP encoding}
Here we recap the ASP encoding scheme given a configuration $Config = (R,F,M,I)$ of legal rules. We will refer to this in the proof of lemma in 5.7, which will be given next.
\begin{lstlisting}[language=Prolog, numbers=left]
% For any f in F, we have:
is_legal(f). 

% All the modifiers get added as facts like for example:
despite(1,2).
subject_to(4,5).

% Any rule r in R is encoded using the general schema:
according_to(r,C_r):-is_legal(pre_con(r)).

% Say {a,b,c} is a minimal inconsistent set in I, then this would get encoded as: 
opposes(a,b) :- is_legal(c)
opposes(a,c) :- is_legal(b).
opposes(b,c) :- is_legal(a).
%The above is done for every minimal inconsistent set. A pair from the set forms the opposes predicate and the rest of the elements go in the body 

% Say {d,e,f,g} is another minimal inconsistent set in I, then this would get encoded as:

opposes(d,e) :- is_legal(f),is_legal(g).
opposes(d,f) :- is_legal(e),is_legal(g).
opposes(d,g) :- is_legal(f),is_legal(e).
opposes(e,f) :- is_legal(d),is_legal(g).
opposes(e,g) :- is_legal(f),is_legal(d).
opposes(f,g) :- is_legal(d),is_legal(e).

% If we had a minimal inconsistent set consisting of only 2 elements say {j,k}, this would get encoded as:

opposes(j,k).

% Opposes is a symmetric relation
opposes(X,Y):-opposes(Y,X).


% Encoding for 'despite'
defeated(R,C,R1) :-
    according_to(R,C), according_to(R1,C1), despite(R,R1).

%Encoding for 'subject_to'
defeated(R,C,R1) :-
    according_to(R,C), legally_valid(R1,C1),
    opposes(C,C1), subject_to(R1,R).

% Encoding for 'strong_subject_to'
defeated(R,C,R1) :-
    according_to(R,C), legally_valid(R1,C1),
    strong_subject_to(R1,R).

not_legally_valid(R) :- defeated(R,C,R1).

legally_valid(R,C):-according_to(R,C),not not_legally_valid(R).

is_legal(C):-legally_valid(R,C).
\end{lstlisting}

\section{Proof of Lemma 4 in 5.7}\label{sec:proofs}

Firstly note, that the converse of the lemma is false. That is, there are configurations and legal models of those configurations that do not correspond to any answer set of the ASP encoding. A simple example of this is the following: Consider the configuration where there are only 2 rules:\\ (1): $a\rightarrow a$\\
(2): $not$ $a\rightarrow b\\$

There are no other facts, modifiers or minimal inconsistent sets. Then for
this configuration $\{legally\_valid(1,a),
is\_legal(a)\}$ is a legal model but it does not correspond to any answer set
of the ASP encoding. As explained in \cite{KRR_notes}, this is essentially due
to the fact that not all minimal supported models of a logic program are
stable models. See \cite{KRR_notes} for the example given above and a further
discussion on this topic.  Now we shall proceed to the proof of the lemma.

Given a configuration $Config$, let $A_{Config}$ be an answer set of it's ASP
encoding and let $S_{A_{Config}}$ be the set of $is\_legal$ and
$legally\_valid$ predicates in $A_{Config}$. It is easy to see that
$A_{Config}$ satisfies A1-A5. For example if the set $M$ from $Config$
contains $strong\_subject\_to(r_{i}, r_{j})$, then $A_{config}$ would contain
$strong\_subject\_to(r_{i}, r_{j})$. Now if $S_{A_{Config}}$ contains
$legally\_valid(r_{i}, C_{r_{i}})$, then so would $A_{Config}$. Now, if
$pre\_con(r_{j})$ is satisfied in $A_{Config}$, then
$according\_to(r_{j},C_{r_{j}})$ is in $A_{Config}$ and therefore
$defeated(r_{j}, C_{r_{j}}, r_{i})$ is in $A_{Config}$ by line 44 of the
general encoding shown above. Therefore $not\_legally\_valid(r_{j})$ is in
$A_{Config}$ by line 48 of the encoding. Therefore by the line 50 of the
encoding, $legally\_valid(r_{j},C_{r_{j}})$ is not in $A_{Config}$. Therefore
$legally\_valid(r_{j},C_{r_{j}})$ is not in $S_{A_{Config}}$.

Now if $pre\_con(r_{j})$ is not satisfied in $A_{Config}$, then
$according\_to(r_{j},C_{r_{j}})$ is not in $A_{Config}$ and so again
$legally\_valid(r_{j},C_{r_{j}})$ is not in $A_{Config}$ and therefore not in
$S_{A_{Config}}$.

We shall now show that $S_{A_{Config}}$ satisfies A6 and A7.

Say the set $M$ contains $subject\_to(r_{i}, r_{j})$ and
$legally\_valid(r_{i}, C_{r_{i}})$ is in $S_{A_{Config}}$. Furthermore suppose
that there exists some $k\in I$ which contains $C_{r_{i}}$ and $C_{r_{j}}$
such that $is\_legal(k\setminus \{C_{r_{j}}\})\subseteq S_{A_{Config}}$. Then
it follows that, $is\_legal(k\setminus \{C_{r_{j}}\})\subseteq
A_{Config}$. Therefore due to the way that the $opposes$ predicates are
defined in the encoding, it follows that $opposes(C_{r_{i}}, C_{r_{j}})$ is in
$A_{Config}$. Now if $pre\_con(r_{j})$ is in $A_{Config}$ then it follows from
line 39 of the encoding that, $defeated(r_{j}, C_{r_{j}}, r_{i}) $ is in
$A_{Config}$, therefore $legally\_valid(r_{j}, C_{r_{j}})$ is not in
$A_{Config}$ and therefore not in $S_{A_{Config}}$.

Again as before, if $pre\_con(r_{j})$ is not in $A_{Config}$ then
$legally\_valid(r_{j}, C_{r_{j}})$ is not in $A_{Config}$ and therefore not in
$S_{A_{Config}}$.

Suppose $S_{A_{Config}}\models pre\_con(r_{j})$, then $A_{Config}$ satisfies
$pre\_con(r_{j})$. So $according\_to(r_{j},C_{j})$ is in $A_{Config}$, then if
$legally\_valid(r_{j}, C_{j})$ is not in $A_{Config}$, according to lines 48
and 50 of the encoding it must be the case that $defeated(r_{j},C_{j},r_{k})$
is in $A_{Config}$ for some rule $r_{k}$. But then because of the way that the
$defeated$ predicate is defined in lines 35, 39, 44, it must mean that rule
$r_{j}$ is invalidated in accordance with either A4, A5 or A6. So
$S_{A_{Config}}$ satisfies A7. $\square$

\section{Pathological rule configuration examples}
In this section we shall briefly give some examples of rule configurations
that fail to satisfy certain properties.

One may suspect that given any configuration, the ASP encoding only generates answer sets corresponding to subset minimal legal models. However this is not the case. Consider the configuration where there are 3 rules:\\ 
$(1)$ $a\rightarrow c$\\
$(2)$ $not$ $c\rightarrow e$\\
$(3)$ $a\rightarrow a$\\
The only fact is $is\_legal(a)$, and there are 2 modifiers $despite(1,2)$, $strong\_subject\_to(3,2)$. There are no minimal inconsistent sets.\\

For this configuration, the ASP encoding generates two answer sets
corresponding to the legal models: \newline $\{is\_legal(a)$,
$legally\_valid(3,a)\}$ and $\{is\_legal(a)$, $legally\_valid(3,a)$,
$is\_legal(c)$, $legally\_valid(1,c)\}$.

We suspect that the ASP encoding does only return subset minimal legal models
if there is no negation as failure in rule pre-conditions or if there are no
$despite$ modifiers, however pursuing this matter fully is left for future
work.

Here we will give an example of a rule configuration that has no legal models
even though none of the rule modifiers involve a rule directly being subject
to itself.

Consider the configuration where there are 2 rules:\\ $(1)$ $a\rightarrow b$\\
$(2)$ $b\rightarrow c$\\
The only fact is $is\_legal(a)$, there is one modifier $subject\_to(2,1)$ and there is one minimal inconsistent set $\{b,c\}$. Then this rule configuration has no legal models. 

\section{Further remarks on example in 5.8}

Here we explore further modifications of the example in 5.8. First, we wish to remind the reader that if there was a 5th rule in this rule set and we had a $despite(4,5)$ modifier, then as long as the precondition of rule 4 is true, it would still invalidate rule 3 even if rule 4 itself got invalidated by rule 5.

However, in the case of $subject\_to$ and $strong\_subject\_to$, the dominating rules needs to be legally valid to invalidate the subordinate rule. 

As an illustration of the previous point say we have a fifth rule which says, if Bob owns a company, he may spend up to 20 million dollars on cars, and we had $despite(4,5)$ as an additional modifier. Suppose now also we have the three facts that Bob is wealthy, Bob is extremely wealthy and Bob owns a company. Then we would get exactly one legal model/answer set in which exactly rule 1, rule 2 and rule 5 were legally valid. So rule 4 would invalidate rule 3 even though it itself is invalidated by rule 5.


\end{document}